\documentclass[conference]{IEEEtran}
\IEEEoverridecommandlockouts

\usepackage{cite}
\usepackage{subcaption}
\usepackage{amsmath,amssymb,amsfonts}
\usepackage{algorithmic}

\usepackage{graphicx}
\usepackage[percent]{overpic}
\usepackage{textcomp}
\usepackage{xcolor}
\def\BibTeX{{\rm B\kern-.05em{\sc i\kern-.025em b}\kern-.08em
    T\kern-.1667em\lower.7ex\hbox{E}\kern-.125emX}}

\usepackage{cite}
\usepackage{amsmath,amssymb,amsfonts}
\usepackage{algorithmic}
\usepackage{graphicx}
\usepackage{textcomp}
\usepackage{xcolor}
\usepackage{subfig}
\usepackage{cite}
\usepackage{amsmath,amssymb,amsfonts}
\usepackage{algorithmic}
\usepackage{graphicx}
\usepackage{textcomp}
\usepackage{xcolor}
\usepackage[ruled,vlined]{algorithm2e}
\usepackage{microtype}

\usepackage{booktabs}
\usepackage{enumitem}
\usepackage{amsthm}

\theoremstyle{plain}
\newtheorem{theorem}{Theorem}[section]
\newtheorem{proposition}[theorem]{Proposition}
\newtheorem{lemma}[theorem]{Lemma}

\theoremstyle{definition}

\newtheorem{assumption}[theorem]{Assumption}
\theoremstyle{remark}
\newtheorem{remark}[theorem]{Remark}

\newcommand{\tj}{\textit{T}_j} 

\usepackage[letterpaper, top=.75in, bottom=1.1in, left=0.75in, right=0.75in]{geometry}

\begin{document}

\title{Estimating Mixture Distributions via Stochastic Mirror Descent

\thanks{This work is supported in part by the U.S. National Science Foundation Institute for Learning-enabled Optimization at Scale (TILOS; NSF CCF-2112665).}}

\author{
\IEEEauthorblockN{Mohammadreza Ahmadypour, Tara Javidi, Farinaz Koushanfar}
\IEEEauthorblockA{\textit{Department of Electrical and Computer Engineering} \\
\textit{University of California San Diego}, San Diego, CA, USA \\
\{mahmadypour, tjavidi, fkoushanfar\}@ucsd.edu}
}

\maketitle
\begin{abstract}
We revisit the classical problem of estimating an unknown distribution from its samples by fitting a mixture model that minimizes cross-entropy loss. Framing the task as a stochastic convex optimization problem over the space of \( M \)-component mixture distributions, we propose a family of estimators derived from the stochastic mirror descent (SMD) algorithm. This optimization-based approach provides a principled and flexible framework that generalizes traditional estimators and proposes a variety of novel estimators through the choice of Bregman divergences.

A key advantage of our method is that it scales efficiently with the number of candidate components \( f_i \); that is, one can employ a large set of basis distributions in the mixture model without incurring significant computational overhead. This enables richer approximations and improved estimation accuracy.

Moreover, in the case of categorical distribution (discrete outcomes) our estimators do not require a strict lower bound, in other words our framework does not require the precise knowledge of the support of the distribution. 

We demonstrate that, under mild conditions, the proposed \( \varphi \)-SMD estimators achieve near-optimal convergence rates in both Kullback–Leibler (KL) divergence and \( \ell_2 \)-norm and offer practical benefits when computation is expensive. Our numerical analysis highlights improved performance guaranties over classical estimators, particularly in terms of sample efficiency and scalability.

\end{abstract}
\section{Introduction}

We study the problem of estimating an unknown probability distribution from sequentially observed data. We consider a setting in which independent samples $\{\zeta_i\}_{i=1}^N$ are drawn from an unknown target distribution $P^*$, which we estimate in an online fashion using a given class of mixture distributions. Performance is measured in terms of the expected Kullback--Leibler (KL) divergence from the optimal distribution within this class.

Estimating probability distributions from data is a foundational problem in statistics and machine learning, spanning both parametric and non-parametric approaches~\cite{fryer1977review, wegman1972nonparametric}. Classical methods include kernel density estimation (KDE), $k$-nearest neighbors ($k$NN), and likelihood-based procedures for parametric distributions such as expectation--maximization (EM)~\cite{wu1983convergence}. While KDE and $k$NN offer flexible nonparametric estimators with asymptotic guarantees~\cite{zhao2022analysis, jiang2017uniform}, they rely on hyper-parameters such as bandwidth or neighborhood size that require careful tuning and batch processing. Similarly, while EM-based methods are fundamental for parametric density estimation, they are inherently designed for batch data. Consequently, these methods are ill-suited for streaming or fully online settings; although incremental variants exist, their theoretical guarantees are often difficult to establish~\cite{same2007online}.

We propose an optimization-based approach that views distribution estimation through the lens of online stochastic convex optimization. Leveraging stochastic mirror descent (SMD)~\cite{nemirovski}, we construct estimators that are updated sequentially using individual observations. By selecting an appropriate distance-generating function, this framework yields a broad class of algorithms, which we refer to as $\phi$-SMD estimators.

Specializing the analysis of stochastic mirror descent, we establish finite-sample convergence guarantees for the proposed $\phi$-SMD estimators. In particular, we show that the expected KL divergence between the estimator and the best-in-class approximation of the true distribution converges at a rate of $O(1/\sqrt{N})$ under general conditions. This rate improves to $O(1/N)$ when the geometry induced by the distance-generating function yields strong convexity, such as in cases corresponding to the KL divergence or the squared $\ell_2$-norm. These results demonstrate that optimization-based online estimators can achieve statistically meaningful rates comparable to classical methods while remaining fully sequential and computationally lightweight.

This formulation enables SMD to operate naturally in a streaming fashion, requiring significantly lower memory as it avoids storing the full dataset. Furthermore, it facilitates structural constraints and regularization through the mirror map. While KDE and $k$NN often suffer from the curse of dimensionality, and EM can be sensitive to initialization, SMD provides stronger scalability and theoretical guarantees under convexity assumptions, making it particularly advantageous in large-scale sequential learning.

To complement our theoretical analysis, we present extensive empirical studies focusing on estimators induced by KL and $\ell_2$ geometries. Our experiments investigate high-dimensional and sparse categorical settings, where classical nonparametric methods and Bayesian smoothing techniques may suffer from slow convergence or hyperparameter sensitivity. We observe that the proposed online estimators exhibit robust performance and favorable convergence behavior, particularly in low-sample regimes.

\textbf{Notations:}We denote by $\Delta_M$ the probability simplex in $\mathbb{R}^M$, defined as
\[\scalebox{0.85}{$
\Delta_M
\;=\;
\left\{
\mathbf{m} \in \mathbb{R}^M \;:\;
\sum_{i=1}^M m_i = 1,
\quad \text{for all } i=1,\dots,M
\right\}.$}
\]
Throughout this paper, we use boldface letters to denote vectors. Moreover, $\mathrm{KL}(\cdot \| \cdot)$ denotes the Kullback--Leibler (KL) divergence.

\section{Problem Statement }
\label{sec:statement}


Let \( P^* \) denote the true but unknown distribution generating samples from a sample space \( \Omega \) defined over an event space \( \mathcal{X} \). Our objective is to design an online estimator that approximates \( P^* \) as new samples arrive. At each time step \( i \), we observe a random variable \( \zeta_i \in \Omega \), drawn independently and identically distributed (i.i.d.) from \( P^* \).

\vspace{0.5em}
\noindent
Our approach views this estimation task through an optimization lens over a kernelized space of densities. We approximate \( P^* \) using an over-complete mixture model of the form
\begin{align}
    Q(\zeta) = \sum_{i=1}^M m_i f_i(\zeta),
\end{align}
where \( \{f_i\}_{i \in [M]} \) represents a possibly redundant dictionary of densities, and \( m = (m_1, \dots, m_M) \in \Delta_M \) denotes the mixture weights constrained to the probability simplex. The dictionary \( \{f_i, i \in [M]\} \) can be quite large.

\vspace{0.5em}
\noindent
We seek to identify the mixture \( Q \) that minimizes the cross-entropy (equivalently, the KL divergence) between \( P^* \) and the mixture model:
\begin{align}
\label{eq:main_problem}
\textbf{(P1)} \qquad 
\min_{m \in \Delta_M} \; \mathbb{E}_{P^*} 
\left[\log \frac{1}{\sum_{i=1}^M m_i f_i(\zeta)} \right].
\end{align}

\vspace{0.5em}
\noindent
This is a convex objective function over a convex and compact set \(\Delta_M\). We denote the unique optimal solution by \( m^* \), and the corresponding optimal mixture approximation by
\begin{align}
    Q^*(\zeta) = \sum_{i=1}^M m^*_i f_i(\zeta).
\end{align}
It is straightforward to verify that \( Q^* \) also minimizes the Kullback–Leibler divergence within the class of mixture models \( D_{\mathrm{KL}}(P^* \,\|\, Q) \).

\vspace{0.5em}
\noindent
To ensure that the optimization problem~\textbf{(P1)} remains strongly convex and that each step of stochastic gradient descent is well defined and finite, we impose the following assumptions.

\begin{assumption}\label{assumption 2}
For all \( i,j \in [M] , \ \forall \ \zeta \in Z \), there exists a constant \( G_\infty > 0 \) such that
\[
\frac{f_i(\zeta)}{f_j(\zeta)} \ge \frac{1}{G_\infty}.
\]
\end{assumption}

This assumption ensures that the average stochastic gradient remains bounded at each iteration and is easy to verify in many practical settings. 

\begin{assumption}\label{assumption 1}
There exist an event \( A \in \mathcal{F} \) with nonzero probability \( P^*(A) > 0 \),  such that \(\{ f_i(\cdot)\big|_{A}\}\) are linearly independent. 
\end{assumption}

\vspace{0.5em}
\noindent
Next we show that under these assumptions, Problem~\textbf{(P1)} is strongly convex in the mixture weights \( m \). This will allow us to use \textbf{Stochastic Mirror Descent (SMD)}.

\begin{lemma} \label{lemma SC}
    Problem \textbf{(P1)} is a stochastic program with a strongly convex objective function. 
\end{lemma}

We provide the proof of this lemma in the Appendix.

\section{Stochastic Mirror Descent Estimators}
\label{sec:estimators}

We will use the Mirror Descent algorithm to optimize the choice of mixture distributions. This method offers great deal of flexibility in choosing geometries and step sizes. 

Our first Algorithm simply specializes the stochastic mirror descent algorithm with distance generating function $\phi(\cdot)$ to iteratively obtain a sequence of $M-$dimensional vector representing the weights , $m^i, i=1,2, \ldots$ and then uses  
 the Cesàro sum of the 
sequence
$\tilde{\mathbf{m}}^{n}= \frac{1}{n}\sum_{i=1}^n\mathbf{m}^i$
to obtain an estimator. The Pseudo-code for Algorithm \ref{alg smd} is given below.

\begin{algorithm}
\caption{$\phi$-SMD Estimator} \label{alg smd}
\textbf{Inputs:} Distance generating function $\phi(\cdot)$, initial estimate $m^0$, sequence of step sizes $\gamma_i$, and observations $\zeta_i \in \Omega $ for $i=1,2, \ldots$ N \\
\textbf{Output:} Estimator  $Q_\phi^{n+1}=\sum_{i=1}^M\tilde{m}^{n+1}_i.f_i$ \\ 
\For{$i = 0$ \textbf{to} $n$}{

    \begin{gather}
    \mathbf{m}^{i+1} = \arg \min_{\mathbf{z} \in \Delta_M} \left[ D_\phi(\mathbf{m}, \mathbf{z}) \right. \ \ \ \ \ \  \ \ \ \ \  \ \ \ \ \ \ \ \ \ \ \ \nonumber \\
    \left.   \ \ \  + \gamma_i   
    \begin{bmatrix}
    \frac{f_1(\zeta)}{\sum_{\tau=1}^Mm^i_\tau f_\tau(\zeta)}, \dots, \frac{f_M(\zeta)}{\sum_{\tau=1}^Mm^i_\tau f_\tau(\zeta)}
    \end{bmatrix} 
    (\mathbf{m} - \mathbf{z})
    \right] \label{eq:updaterule}  
    \end{gather}
    \begin{gather*}
        \tilde{\mathbf{m}}^i = \frac{\mathbf{m}^i}{i}+\frac{i-1}{i}\tilde{\mathbf{m}}^{i-1}
    \end{gather*}
}
\end{algorithm}

The algorithm \ref{alg smd} remains effective even in the presence of noisy observations, provided that the noise is unbiased. In such cases, the algorithm can still be reliably applied.

By expressing  $\phi=\frac{1}{2}\|\mathbf{m}\|_2^2$ as  our distance generating function, the resulting Bregman distance simplifies to the $\ell_2$ distance. and SMD reduces to a simple Stochastic Gradient Descent as can be seen in Algorithm \ref{alg sgd}:

\begin{algorithm}
\caption{SGD Estimator} \label{alg sgd}
\textbf{Inputs:} Initial estimate $m^0$, sequence of step sizes $\gamma_i$, and observations $\zeta_i  \in \Omega$ for $i=1,2, \ldots$ N \\
\textbf{Output:} Estimator $Q_{SGD}^N =\sum_{i=1}^M m^{N+1}_i.f_i$ \\ 
\For{$i = 0$ \textbf{to} $n$}{

    \begin{gather}\label{eq sgd}
    \mathbf{m}^{i+1} = \\ \nonumber 
    \Pi_{\Delta_M} \left(m^{i} - \gamma_i\begin{bmatrix}
    \frac{f_1(\zeta)}{\sum_{\tau=1}^Mm^i_\tau f_\tau(\zeta)}, \dots, \frac{f_M(\zeta)}{\sum_{\tau=1}^Mm^i_\tau f_\tau(\zeta)}
    \end{bmatrix}\right)
    \end{gather}
}
\end{algorithm}

When \( \phi \) is selected to be \( \sum_{i=1}^M m_i \log m_i \), the corresponding Bregman distance reduces to the Kullback-Leibler (KL) divergence. Interestingly, this results in a closed form which resembles a Bayesian update but with an exponential form. Again below we provide pseudo-code for Algorithm~\ref{alg weak}:

\begin{algorithm}
\caption{Exp-SMD Estimator} \label{alg weak}
\textbf{Inputs:} $\forall i\in [1,\dots,n], \zeta_i  \in \Omega$, Prior Distribution $m^0$ \\
\textbf{Output:} Estimator  $Q_{KL}^N = \sum_{i=1}^Mm^{N+1}_i.f_i$ \\
\For{$i = 0$ \textbf{to} $n$}{
            \begin{align*}
                m^{i+1}_j= \frac{m^{i}_j \ \ e^{\gamma_i \ \ \frac{f_j(\zeta)}{\sum_{\tau=1}^Mm^i_\tau f_\tau(\zeta)}}}{\sum_{\tau'=1}^{M}\left(m^{i}_{\tau'} \ \ e^{\gamma_{i} \ \ \frac{f_{\tau'}(\zeta)}{\sum_{\tau=1}^Mm^i_\tau f_\tau(\zeta)}}\right)}
            \end{align*}
        
        }
\end{algorithm}

Notice that, we have the flexibility to choose step-size $\gamma_i$ as a hyper parameter controlling the algorithm dynamic hence the rate of convergence. Extensive discussion regarding good choices of step size is provided in\cite{painless}. However, in this paper, we leave the question of how to best select varying step sizes as an important direction for future work.


\section{Main Results: Convergence Rates}
\label{sec:results}

As a direct corollary of Equation~(2.46) in~\cite{nemirovski}, we obtain the following proposition by defining
\begin{align} \label{delta_over_S}
R_{\phi}(\mathcal{S}) := \left[ \max_{z \in \mathcal{S}} \phi(z) - \min_{z \in \mathcal{S}} \phi(z) \right]^{1/2},
\end{align}
and utilizing the fact that the objective function is convex, we can obtain a general upper bound on the performance of our general estimator after observing N samples:

\begin{proposition}\label{prop}
Let the step size, in \textbf{Algorithm~\ref{alg smd}}, be set to 
\(
\gamma_i = \frac{ R_\phi(\Delta_M)}{ G_\infty \sqrt{N}}, i = 1, \dots, N.
\)
Then, 
\begin{align} \label{eq weak}
\mathbb{E} \left[\log\frac{Q^*}{ Q_{\phi}^N} \right] \leq \frac{R_\phi(\Delta_M) G_\infty}{\sqrt{N}}.
\end{align}
\end{proposition}

Next, we strengthen our analysis by deriving two theorems for specific choices of the distance-generating function: \( \phi(x) = \sum x \log x \) (corresponding to KL divergence) and \( \phi(x) = \frac{1}{2}\|x\|_2^2 \) (corresponding to Euclidean distance), while exploiting the strong-convexity of our objective function gives us some flexibility to improve the convergence bound.

\begin{theorem} \label{thm1}
    Let $Q_{SGD}^N=\sum_{i=1}^M m^{N+1}_i.f_i$ be the output of 
    \textbf{Algorithm~\ref{alg sgd}} when step size $\gamma_i = \frac{2}{\nu(i+1)}$ where \(\nu := \lambda_{\min}\) is our strongly convex parameter .. The mean squared error is upper bounded by: 
\begin{align}\label{eq4}
    \lefteqn{\mathbb{E} \left[ \|\mathbf{m}^N- \mathbf{m}^*\|_2^2 \right]} \\ \nonumber &\leq \frac{1}{N+1}\,
\max\!\left\{\, R_{\ell_2}(\Delta_M),\ \frac{8\,G_\infty^2}{\nu^2} \,\right\}
\end{align}
\end{theorem}

\begin{theorem}\label{thm2}
    Let $Q_{KL}^N=\sum_{i=1}^M m^{N+1}_i.f_i$ be the output of \textbf{Algorithm \ref{alg weak}} when step size $\gamma_i = \frac{2}{\nu(i+1)}$ where \(\nu := \lambda_{\min}\) is our strongly convex parameter. The expected error measured in KL-divergence is given as :
\begin{align}
\lefteqn{ \mathbb{E}_{P^*}\left[ KL(\mathbf{m}^*||\mathbf{m}^N )\right] } \nonumber \\ & \leq \frac{1}{N+1}\, \max\!\left\{ R_{KL}(\Delta_M),\ \frac{2 G_\infty^2}{\nu^2} \right\}
\end{align}

\end{theorem}

\section{Numerical Experiments}
\label{numerical}

In this section, we present a series of numerical experiments designed to evaluate the performance and practical advantages of our proposed optimization-based density estimation framework. Our primary focus is on Algorithm~\ref{alg smd} (Exp-SMD), which operates in an online setting utilizing stochastic updates and KL-geometry. 

To demonstrate the robustness of Exp-SMD, we evaluate our algorithm using synthetic distributions. We begin with a sparse categorical distribution to establish a simple, intuitive baseline for convergence behavior. We then evaluate continuous distributions to test the algorithm's ability to capture multi-scale features, sharp boundaries, and complex topologies.

Across the continuous experiments, we benchmark our method against classical approaches: kernel density estimation (KDE), $k$-nearest neighbors (k-NN), and stochastic gradient descent (SGD).

\subsubsection{Experimental Setup and Hyperparameters}

Before detailing the results, we define the target distributions, the multi-scale kernel dictionary, and the specific parameters used for our solvers and baselines.

\paragraph{1. Target Continuous Distribution}
For our primary continuous evaluation, the objective is to estimate an unknown two-dimensional probability density function (PDF), $P^*(\zeta)$, using a stochastic stream of i.i.d. samples $\zeta_t \sim P^*$. The target distribution operates over $\zeta_x \in [-5, 5]$ and $\zeta_y \in [-5, 5]$, and is an equal mixture of four distinct topological modes:
$$
\scalebox{0.85}{$\displaystyle P^*(x, y) = \frac{1}{4} P_{\text{Donut}} + \frac{1}{4} P_{\text{Square}} + \frac{1}{4} P_{\text{Diagonal}} + \frac{1}{4} P_{\text{Spike}} $}
$$

\begin{itemize}
    \item \textbf{The Donut:} A non-convex ring centered at the origin, defined by a Gaussian profile along its radius.
    $$
    \scalebox{0.85}{$\displaystyle P_{\text{Donut}}(x, y) \propto \exp \left( - \frac{(\sqrt{x^2 + y^2} - 2.5)^2}{2(0.2)^2} \right) $}
    $$
    
    \item \textbf{The Square:} A uniform distribution over a rectangular region, representing sharp, discontinuous boundaries.
    $$
    \scalebox{0.85}{$\displaystyle P_{\text{Square}}(x, y) \propto \mathbb{I}(-2.75 \le x \le -1.25) \cdot \mathbb{I}(1.25 \le y \le 2.75) $}
    $$
    
    \item \textbf{The Diagonal Gaussian:} A highly correlated Gaussian distribution oriented along the line $y=x$, featuring a cross-term to induce strong correlation.
    \begin{align*}
    &\scalebox{0.85}{$ P_{\text{Diagonal}}(x, y) \propto  $} \\
     &\scalebox{0.85}{$\exp \left( - \left[ 2(x - 2.5)^2 - 3.5(x - 2.5)(y - 2.5) + 2(y - 2.5)^2 \right] \right)$} 
    \end{align*}
    
    \item \textbf{The Spike:} An extremely low-variance Gaussian acting as a high-frequency feature (a Dirac delta approximation).
    $$
    \scalebox{0.85}{$\displaystyle P_{\text{Spike}}(x, y) \propto \exp \left( - \frac{(x + 2)^2 + (y + 2)^2}{2(0.1)^2} \right) $}
    $$
\end{itemize}

\paragraph{2. Multi-Scale Kernel Dictionary}
To capture both broad global structures and fine details, we construct a multi-scale kernel dictionary $\{f_i\}$ consisting of isotropic Gaussian kernels:
\begin{equation}
    \scalebox{0.85}{$\displaystyle f_i(\zeta; \mu_i, \sigma_i) = \frac{1}{2\pi\sigma_i^2} \exp \left( -\frac{\|\zeta - \mu_i\|^2}{2\sigma_i^2} \right) $}
\end{equation}

The dictionary is structured in three layers:
\begin{itemize}
    \item \textbf{Coarse Layer:} An $8 \times 8$ grid with $\sigma = 1.5$ (captures global structure).
    \item \textbf{Medium Layer:} A $15 \times 15$ grid with $\sigma = 0.5$ (captures intermediate-scale geometry).
    \item \textbf{Fine Layer:} A $30 \times 30$ grid with $\sigma = 0.15$ (resolves sharp features like the spike and square boundaries).
\end{itemize}

\paragraph{3. Methods and Baseline Parameters}
\begin{itemize}
    \item \textbf{Proposed Method (SMD):} Uses a step-size schedule $\gamma_t = \gamma_0 / (1+t)^{\text{decay}}$ to allow large initial steps and reduced variance in later iterations. We set $\gamma_0 = 0.1$ and $\text{decay} = 0.35$.
    \item \textbf{Method 2 (SGD):} Operates in Euclidean geometry using a softmax parameterization over unconstrained logits $w \in \mathbb{R}^K$ to enforce simplex constraints. Optimized using standard stochastic gradient descent.
    \item \textbf{Method 3 (Auto-Tuned KDE):} Places a Gaussian kernel on each observed data point. The optimal bandwidth $h$ is determined via grid search with cross-validation over a log-spaced range $h \in [0.01, 10.0]$, maximizing the log-likelihood of held-out data.
    \item \textbf{Method 4 (k-NN):} Density is estimated as $\hat{P}(x) \propto 1 / (\pi (r_k(x)^2 + \delta))$, where $r_k(x)$ is the distance to the $k$-th nearest neighbor and $\delta$ is a small constant for numerical stability.
\end{itemize}

\subsubsection{Results and Discussion}

\paragraph{Experiment 1: Sparse Categorical Distribution}
We first evaluate convergence behavior on a sparse categorical distribution. As illustrated in Figure~\ref{fig:1}, our SMD estimator converges significantly faster than the add-a-constant baseline and achieves a lower KL divergence in the long term. 

\begin{figure}[htbp]
    \centering
    \begin{overpic}[width=1\linewidth]{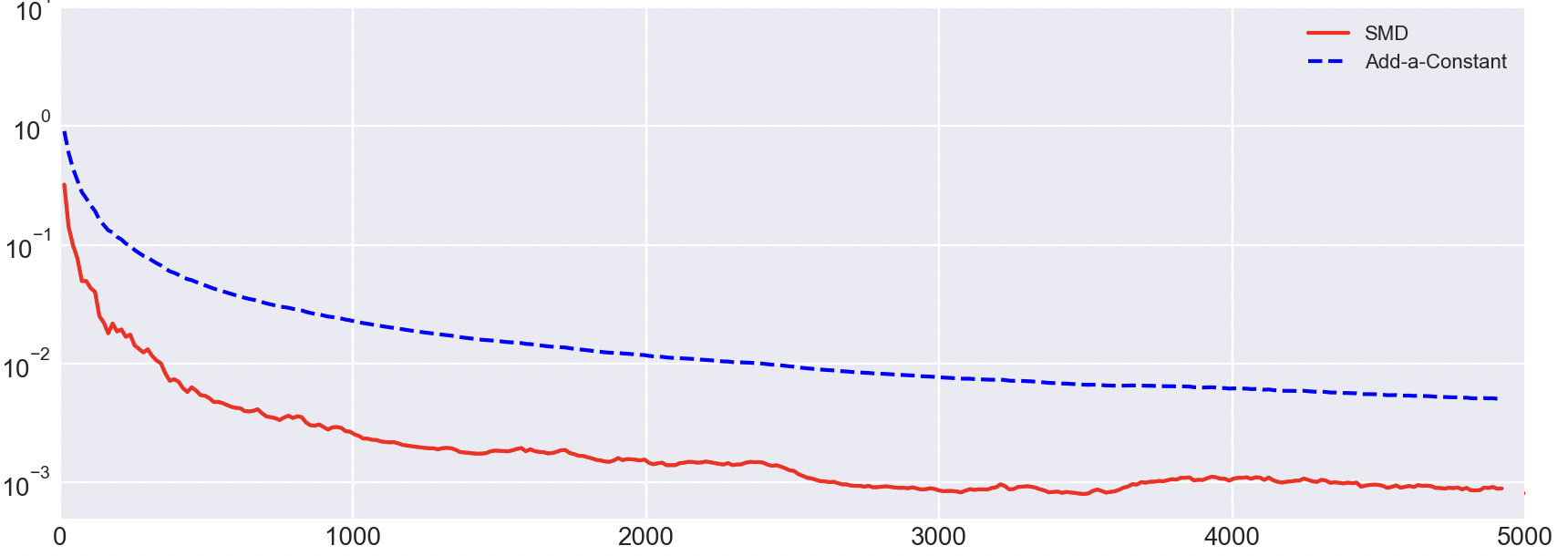}
        \put(50,-3){\small Iteration}
        \put(-3,1){\rotatebox{90}{\smaller Log-scale  KL Divergence}}
    \end{overpic}
    \vspace{1em}
    \caption{Log-scale KL divergence versus iteration for different estimators.}
    \label{fig:1}
\end{figure}

\paragraph{Experiment 2: Four-Mode Continuous Distribution}
Next, we evaluate the four-mode continuous distribution (the donut, square, diagonal, and spike). Figure~\ref{fig:donut} demonstrates the performance of SMD against KDE and k-NN. Our approach accurately captures both the broad, non-convex geometry of the donut and the sharp, localized features of the spike and square, outperforming the baselines in overall estimation accuracy.

\begin{figure}[htbp]
    \centering
    
    \begin{subfigure}[b]{0.15\textwidth}
        \centering
        \includegraphics[width=\textwidth]{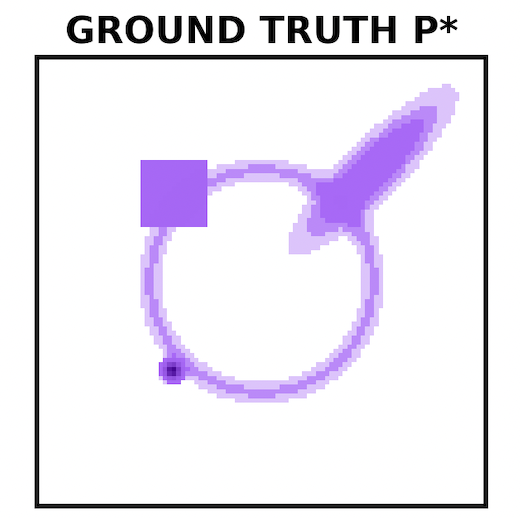}
        \label{fig:sub1}
    \end{subfigure}
    \hfill
    \begin{subfigure}[b]{0.15\textwidth}
        \centering
        \includegraphics[width=\textwidth]{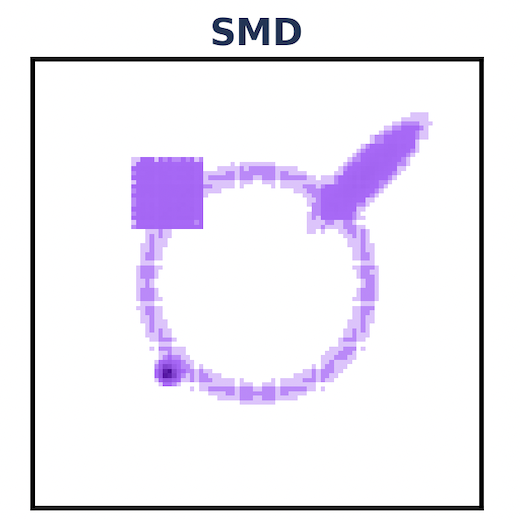}
        \label{fig:sub2}
    \end{subfigure}
    \hfill
    \begin{subfigure}[b]{0.15\textwidth}
        \centering
        \includegraphics[width=\textwidth]{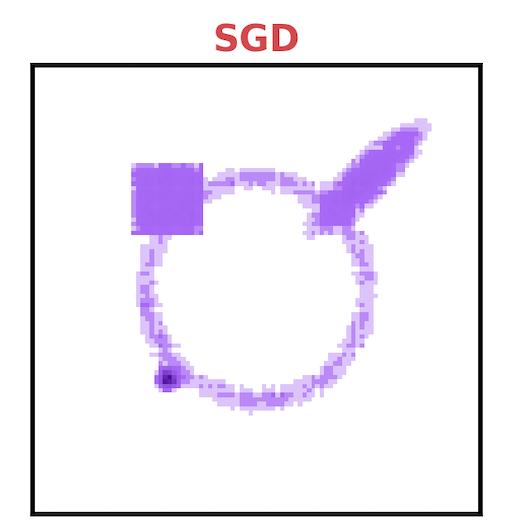}
        \label{fig:sub3}
    \end{subfigure}
    
    \vspace{1em}
    
    \begin{subfigure}[b]{0.15\textwidth}
        \centering
        \includegraphics[width=\textwidth]{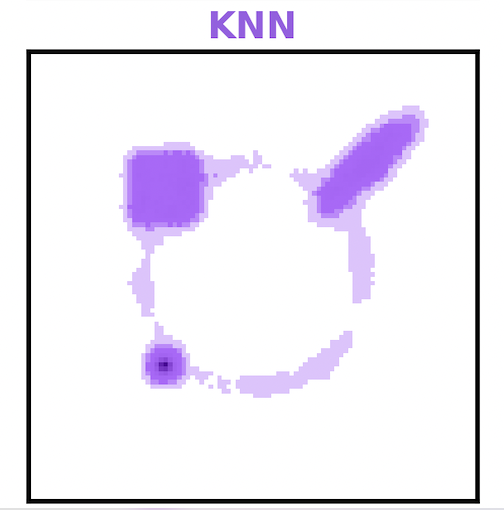}
        \label{fig:sub4}
    \end{subfigure}
    \hfill
    \begin{subfigure}[b]{0.16\textwidth}
        \centering
        \includegraphics[width=\textwidth]{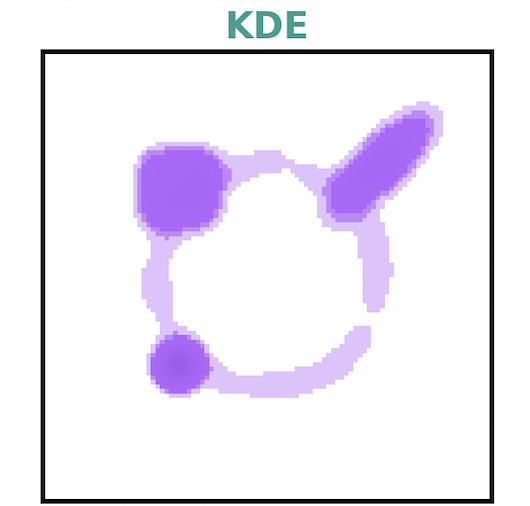}
        \label{fig:sub_table}
    \end{subfigure}
    \hfill
    \begin{subfigure}[b]{0.15\textwidth}
        \centering
        \includegraphics[width=\textwidth]{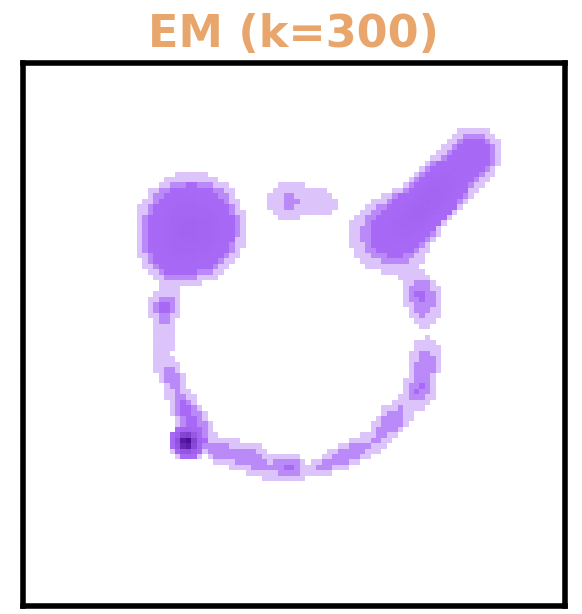}
        \label{fig:subem300}
    \end{subfigure}
    
    \vspace{1em}
    
    \begin{subfigure}[b]{0.3\textwidth}
        \centering
        \includegraphics[width=\textwidth]{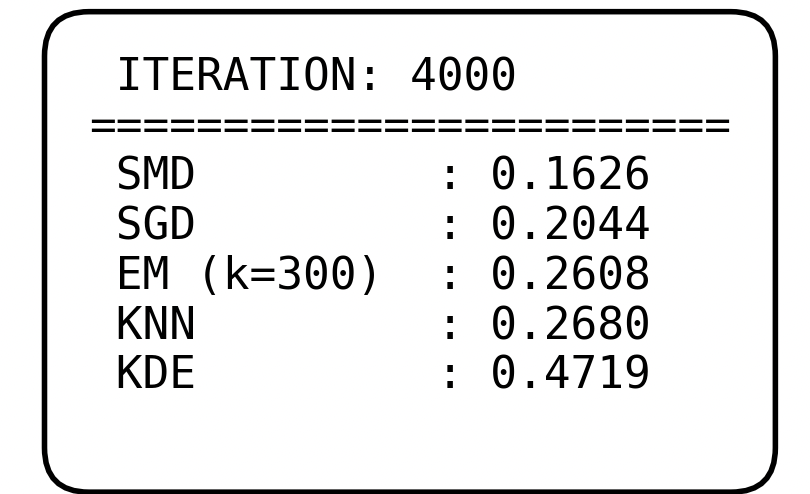}
        \label{fig:sub8}
    \end{subfigure}
    
    \vspace{1em}
    
    \begin{subfigure}[b]{0.5\textwidth}
        \centering
        \begin{overpic}[width=1\linewidth]{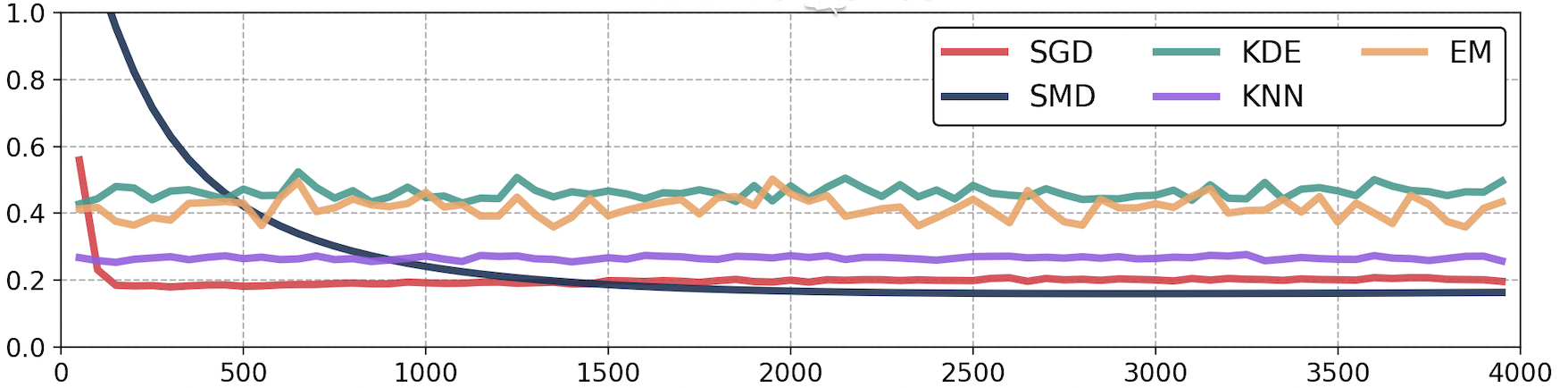}
        \put(45,-3){\smaller Iteration}
        \put(-3,3){\rotatebox{90}{\smaller KL Divergence}}
    \end{overpic}
        \label{fig:sub9}
    \end{subfigure}
    
    \caption{Performance comparison of SMD, EM, KDE, and k-NN on continuous distributions including a donut, Gaussian mixture, and uniform distribution.}
    \label{fig:donut}
\end{figure}

\paragraph{Experiment 3: Multi-Scale Gaussian Mixture}
To further test the resolution of competing methods, we consider a distribution consisting of one wide Gaussian superimposed with four sharp spikes. As shown in Figure~\ref{fig:mix_bandwidth}, our method significantly outperforms the other approaches in this mixed-scale setting.

\begin{figure}[htbp]
    \centering
    
    \begin{subfigure}[b]{0.25\textwidth}
        \centering
        \includegraphics[width=\textwidth]{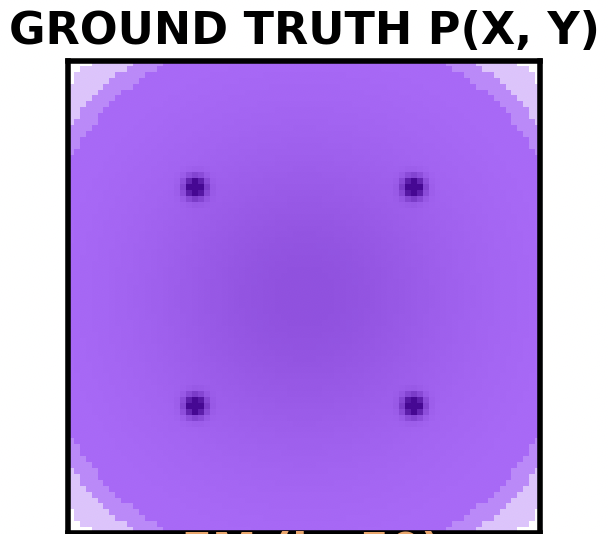}
    \end{subfigure}
    \hfill
    \begin{subfigure}[b]{0.21\textwidth}
        \centering
        \includegraphics[width=\textwidth]{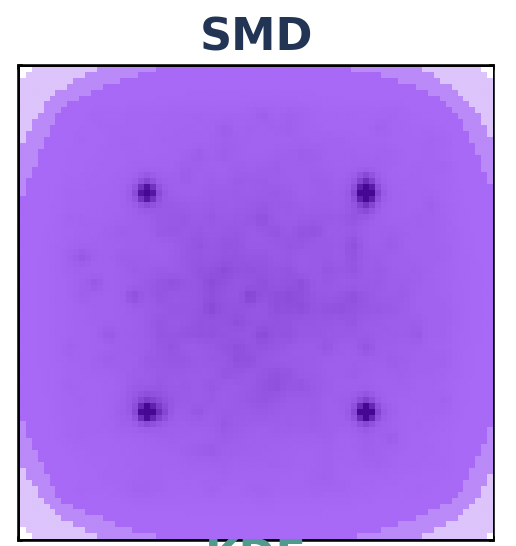}
    \end{subfigure}
    \hfill
    \begin{subfigure}[b]{0.15\textwidth}
        \centering
        \includegraphics[width=\textwidth]{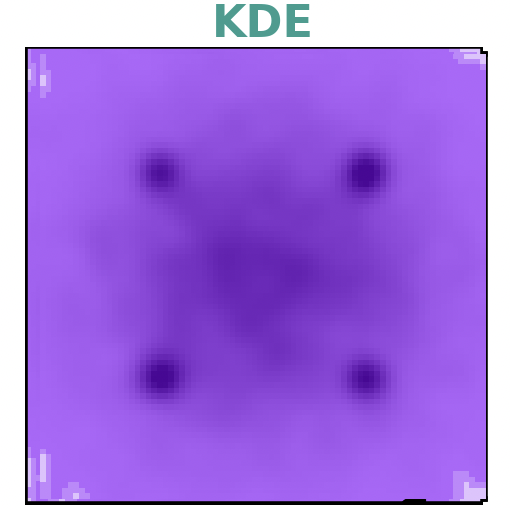}
    \end{subfigure}
    \hfill
    \begin{subfigure}[b]{0.16\textwidth}
        \centering
        \includegraphics[width=\textwidth]{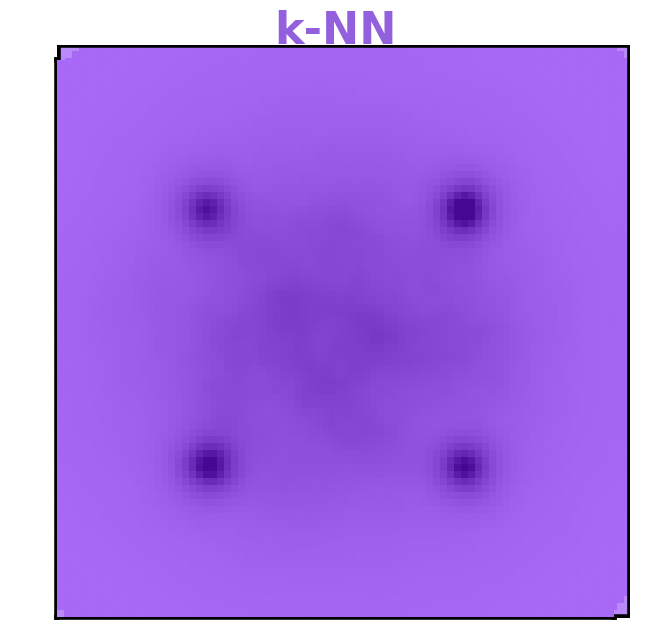}
    \end{subfigure}
    \hfill
    \begin{subfigure}[b]{0.16\textwidth}
        \centering
        \includegraphics[width=\textwidth]{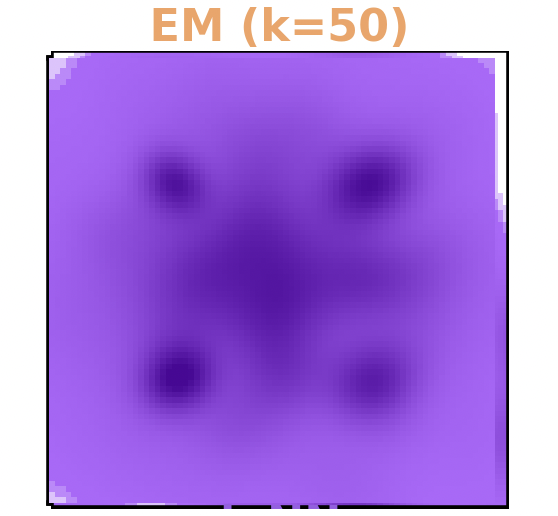}
    \end{subfigure}
    
    \vspace{1em}
    
    \begin{subfigure}[b]{0.4\textwidth}
        \centering
        \begin{overpic}[width=1\linewidth]{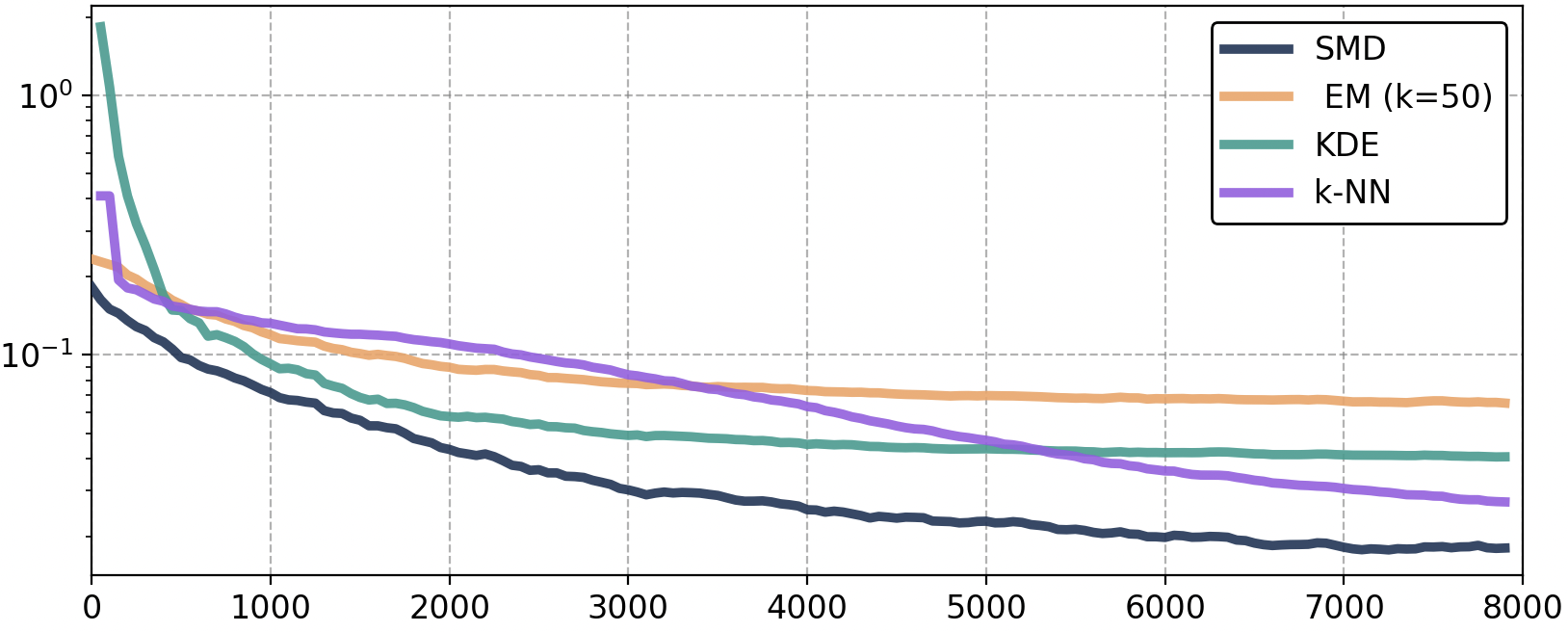}
        \put(45,-3){\smaller Iteration}
        \put(-5,1){\rotatebox{90}{\smaller KL Divergence (log scale)}}
    \end{overpic}
        \label{fig:sub5}
    \end{subfigure}
    \caption{Comparison of SMD, EM, KDE, and k-NN on a mixture of continuous distributions.}
    \label{fig:mix_bandwidth}
\end{figure}

\paragraph{Ablation Study}
To evaluate the computational scalability and algorithmic efficiency of our approach, we conducted an ablation study with a primary focus on processing time and convergence behavior. When executing the framework on a standard CPU, we observed that processing times are highly sensitive to the chosen kernel size. As the kernel size increases, the computational overhead required to evaluate the spatial interactions across the synthetic distributions scales accordingly. Larger kernels necessitate a significantly higher volume of distance calculations and exponential evaluations per iteration, leading to a steep increase in CPU consumption and a corresponding expansion of the overall runtime as can be seen in figure \ref{fig:cpu}.

\begin{figure}[htbp]
    \centering
    
    \begin{subfigure}[b]{0.4\textwidth}
        \centering
        \includegraphics[width=\textwidth]{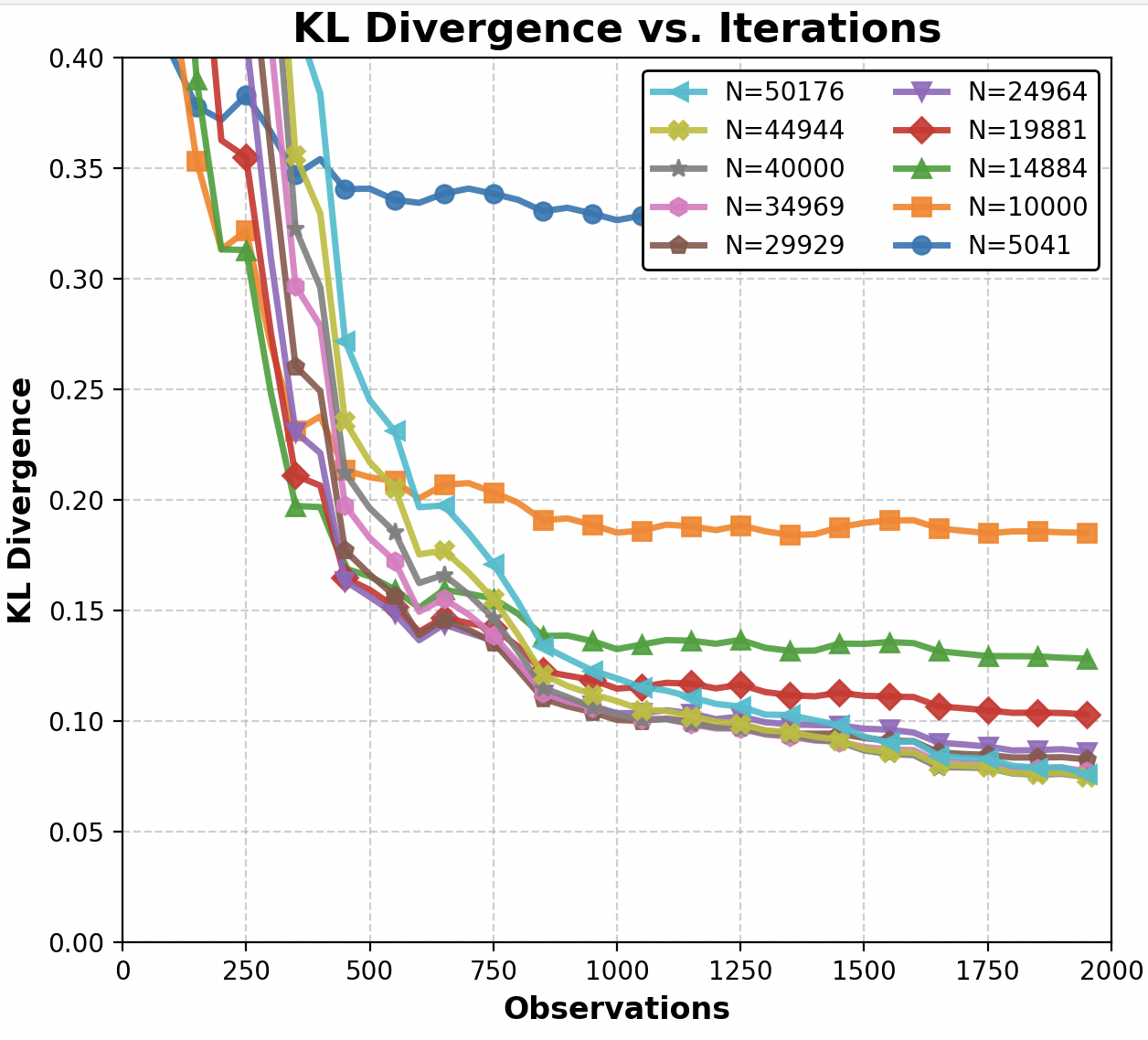}
    \end{subfigure}
    \hfill
    
    \begin{subfigure}[b]{0.4\textwidth}
        \centering
        \includegraphics[width=\textwidth]{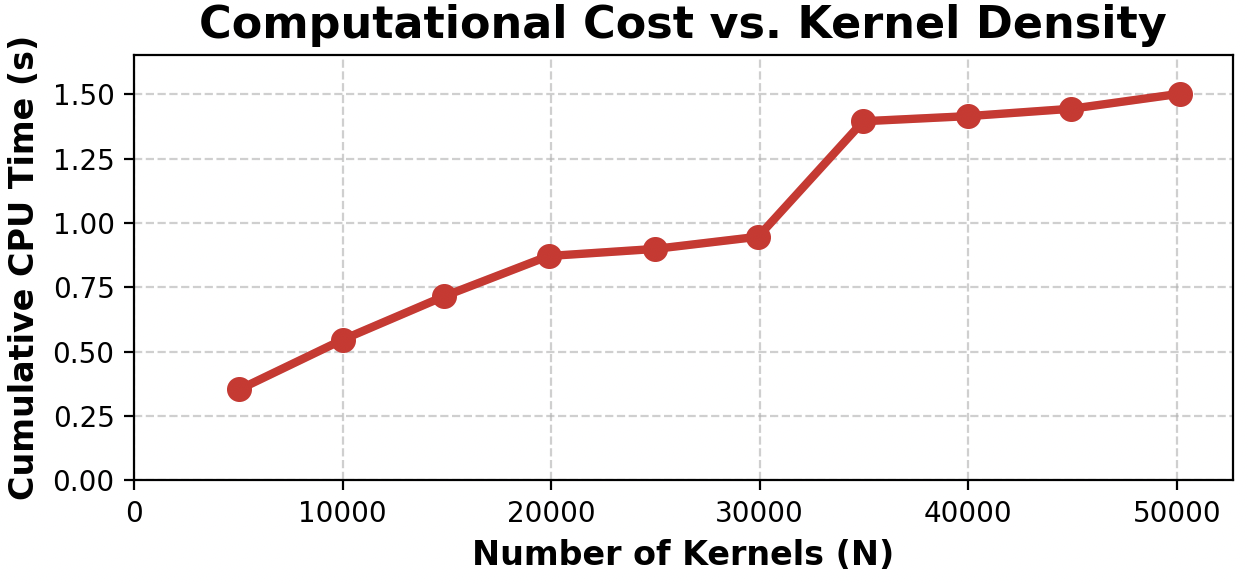}
    \end{subfigure}
    \hfill
    \begin{subfigure}[b]{0.5\textwidth}
        \centering
        \includegraphics[width=\textwidth]{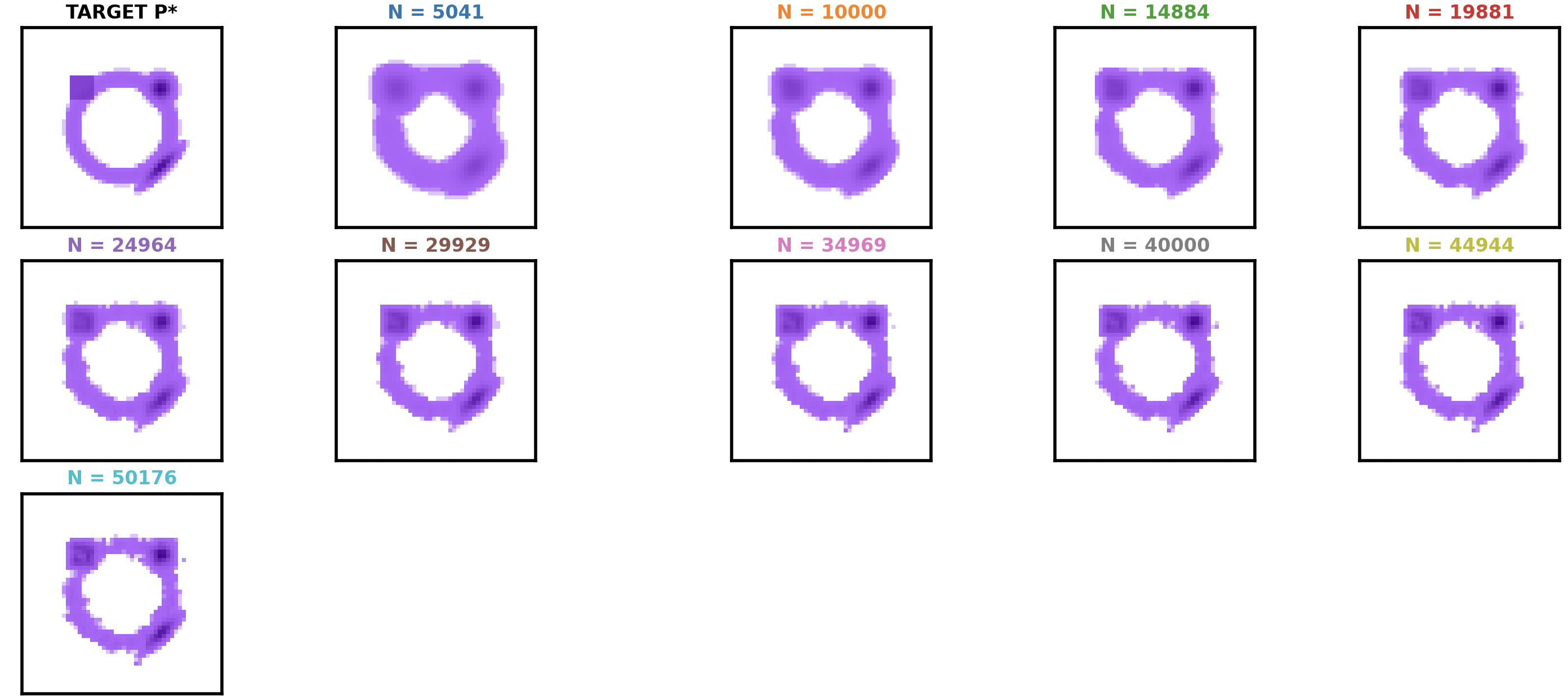}
    \end{subfigure}
    
    \vspace{1em}
    \caption{Comparison of CPU time consumption for different kernel sizes.}
    \label{fig:cpu}
\end{figure}

\begin{figure}[htbp]
    \centering
  
\begin{subfigure}[b]{0.4\textwidth}

        \includegraphics[width=\textwidth]{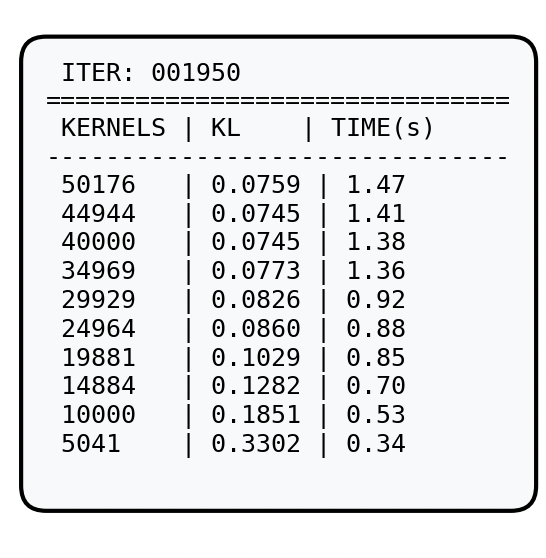}
    \end{subfigure}

\begin{subfigure}[b]{0.395\textwidth}

        \includegraphics[width=\textwidth]{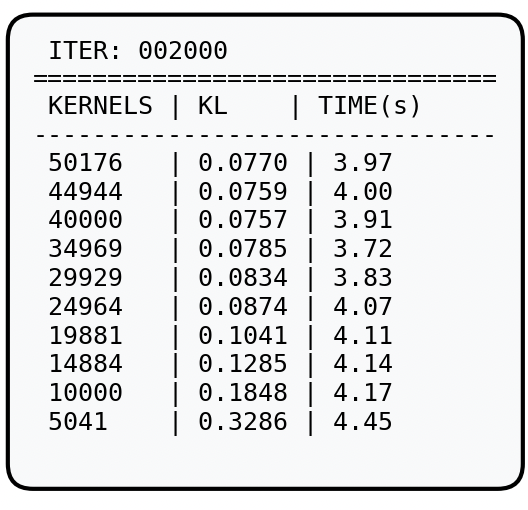}
    \end{subfigure}
\vspace{1em}
    \caption{Comparison of GPU time consumption (top) for different kernel sizes with CPU time consumption (bottom).}
    \label{fig:cpu vs gpu}
\end{figure}

To address this computational demand, we transitioned the evaluation to a GPU-accelerated environment, which yielded substantial improvements across all configurations, as can be seen in \ref{fig:gpu}. The core operations of the optimization steps—particularly the dense matrix multiplications and element-wise kernel evaluations—are highly parallelizable and map exceptionally well to GPU architectures. By offloading these calculations, the GPU efficiently manages the computational complexity associated with larger kernels, resulting in drastically lower latency per iteration and allowing the framework to execute at a much higher performance level.

\begin{figure}[htbp]
    \centering
    
    \begin{subfigure}[b]{0.4\textwidth}
        \centering
        \includegraphics[width=\textwidth]{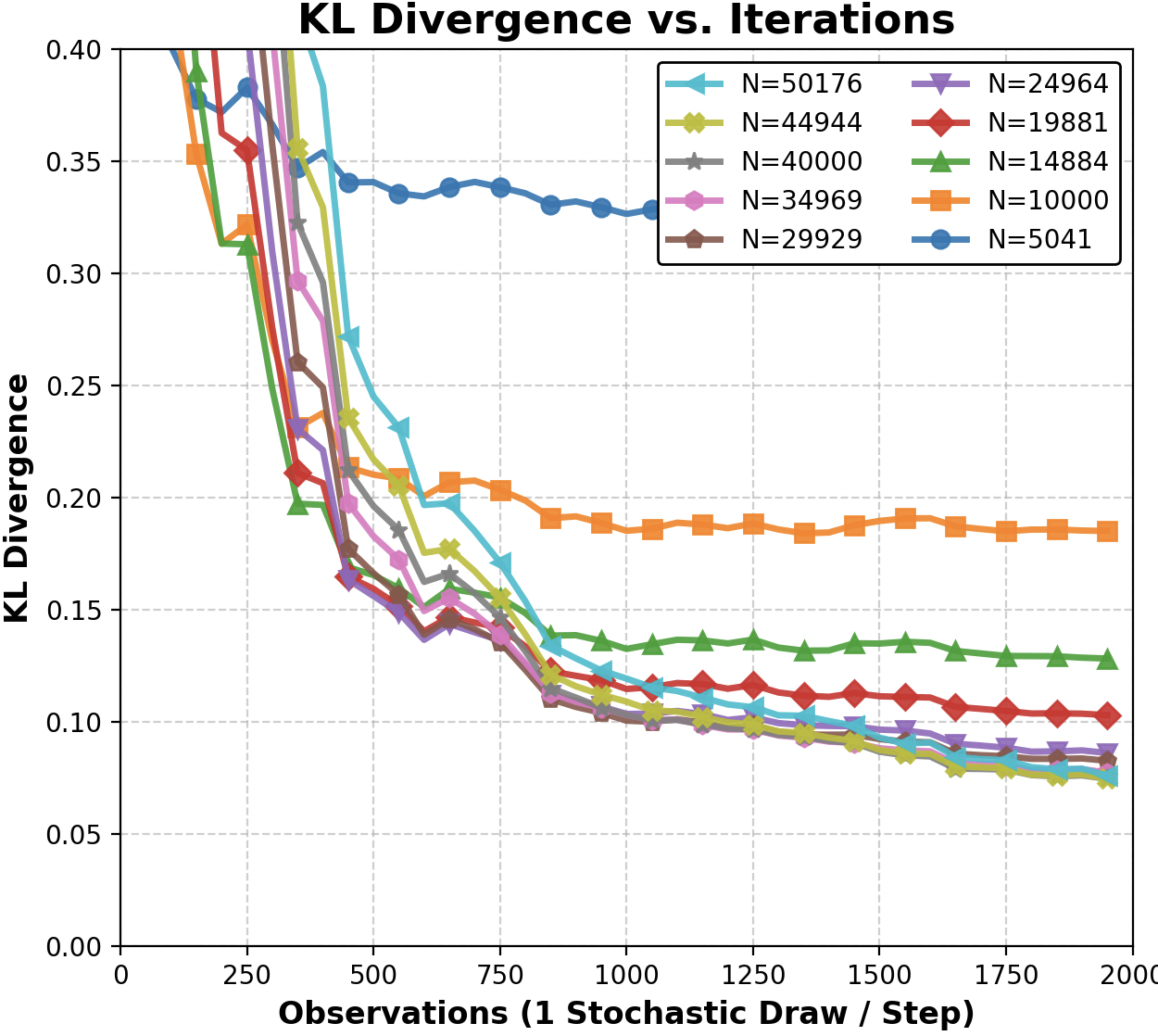}
    \end{subfigure}
    \hfill
    \begin{subfigure}[b]{0.4\textwidth}
        \centering
        \includegraphics[width=\textwidth]{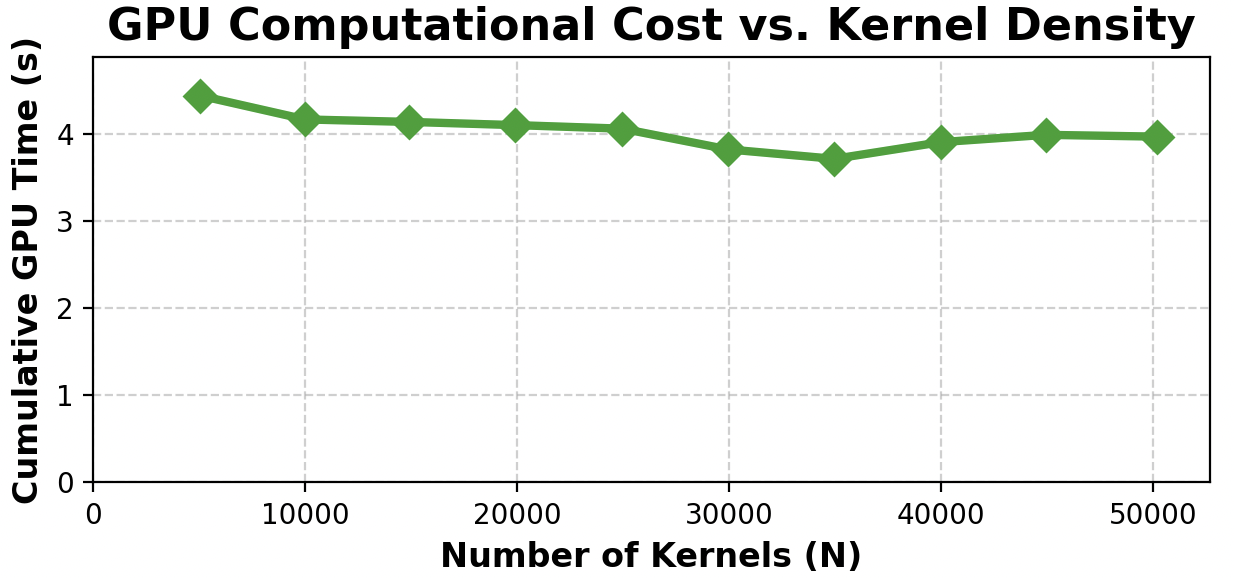}
    \end{subfigure}
    \hfill
    \begin{subfigure}[b]{0.5\textwidth}
        \centering
        \includegraphics[width=\textwidth]{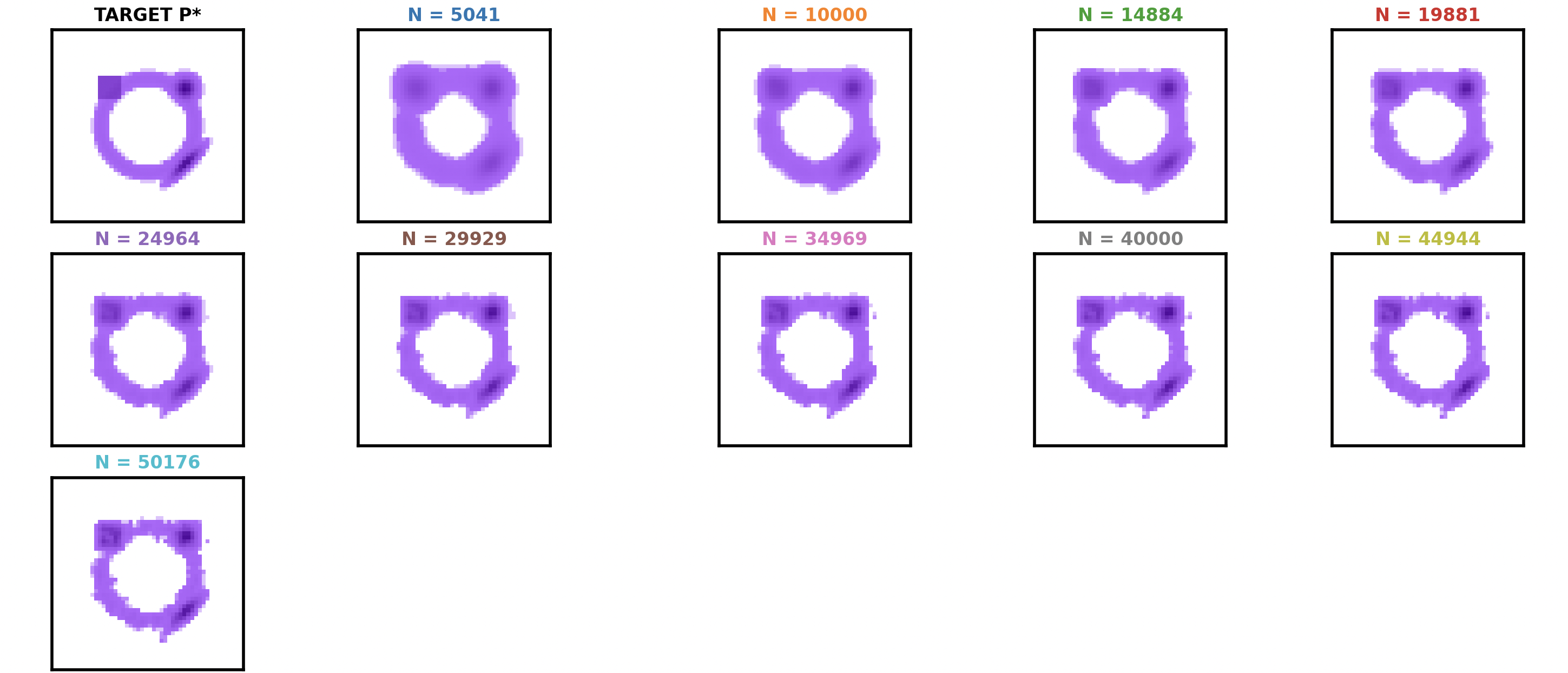}
    \end{subfigure}
    
    \vspace{1em}
    \caption{Comparison of GPU time consumption for different kernel sizes.}
    \label{fig:gpu}
\end{figure}

Finally, the ablation study compared the optimization trajectories of Stochastic Mirror Descent (SMD) and Method 2 (SGD). The empirical results demonstrated that the iteration times and fundamental update mechanics between the two approaches are remarkably similar figure \ref{fig:smdvssgd}. However, while they share this underlying similarity, SMD ultimately proves to be the better method for this problem space; because SMD respects the underlying geometry of the probability distributions during the update steps, it navigates the optimization landscape more effectively and achieves superior results compared to the straightforward gradient updates of SGD.

\begin{figure}[htbp]
    \centering
    
    \begin{subfigure}[b]{0.5\textwidth}
        \centering
        \includegraphics[width=\textwidth]{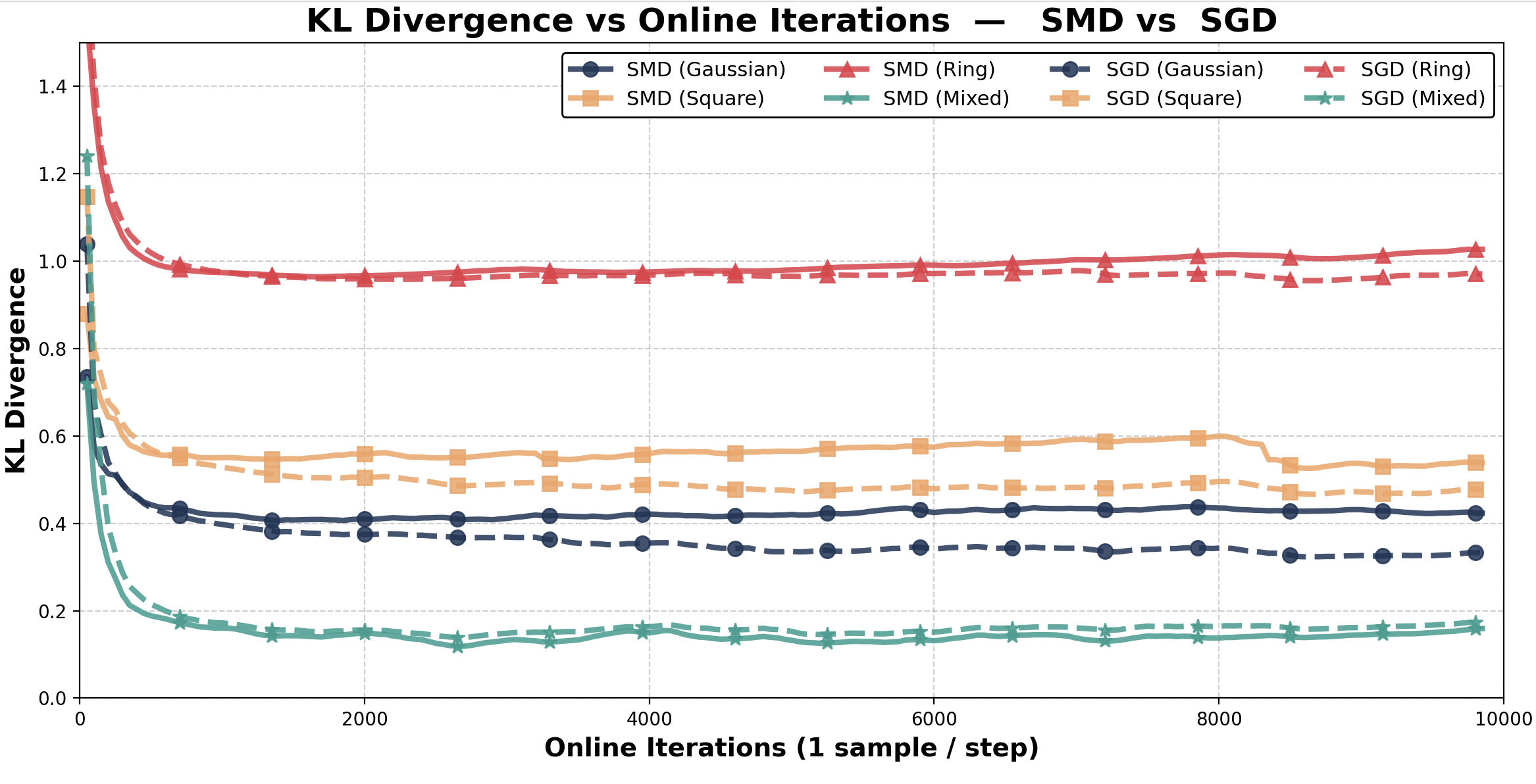}
    \end{subfigure}
    \hfill
    \begin{subfigure}[b]{0.3\textwidth}
        \centering
        \includegraphics[width=\textwidth]{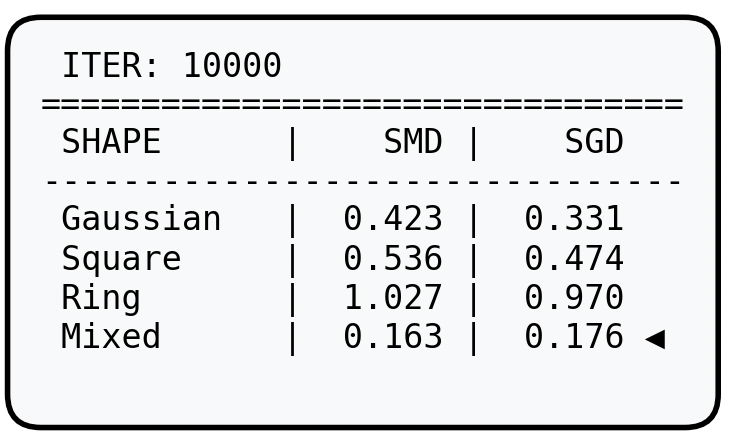}
    \end{subfigure}
    \hfill
    \begin{subfigure}[b]{0.5\textwidth}
        \centering
        \includegraphics[width=\textwidth]{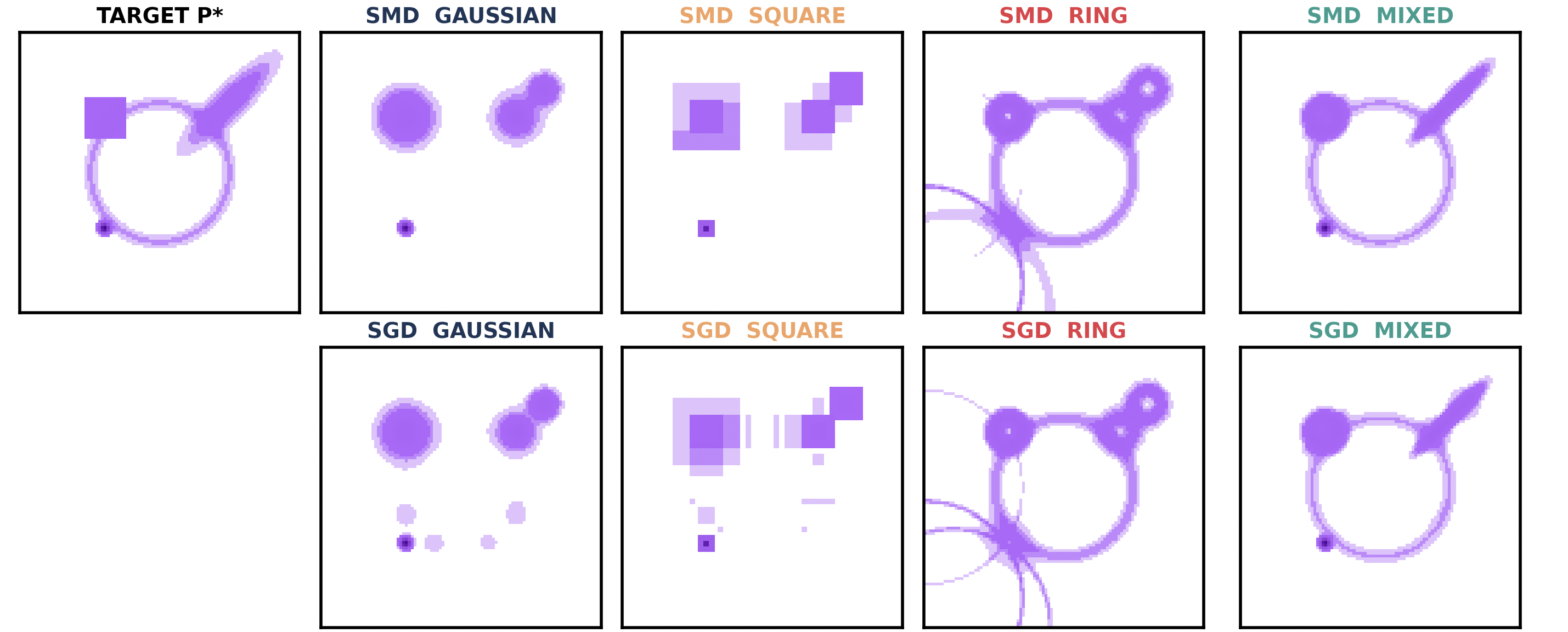}
    \end{subfigure}
    \hfill

    \vspace{1em}
    \caption{Comparison of SMD and SGD over different kernel shapes and sizes.}
    \label{fig:smdvssgd}
\end{figure}

\paragraph{Key Advantages}
Based on these experiments, two primary advantages of our Exp-SMD framework emerge:
\begin{enumerate}
    \item \textbf{Adaptive Resolution:} Unlike KDE, which relies on a single global bandwidth and struggles to balance narrow and broad structures simultaneously (Fig.~\ref{fig:mix_bandwidth}), our multi-scale kernel approach adapts automatically to diverse topological features.
    \item \textbf{Efficient Pruning:} In contrast to alternative fixed-component models that struggle on mixed-scale or non-Gaussian structures, our method automatically assigns negligible weights to irrelevant kernels, effectively pruning unnecessary components from the dictionary without the need for manual tuning.
\end{enumerate}

\section{Conclusion and Future Work}
\label{sec:conclusion}
In this work, we proposed a new framework for distribution estimation via stochastic mirror descent (SMD), offering a flexible and theoretically grounded alternative to traditional approaches such as empirical or Bayesian estimators. By casting the estimation task as a strongly convex optimization problem, we derived convergence bounds for several choices of distance-generating functions, including KL and Euclidean divergences. Our results demonstrate that the resulting $\varphi$-SMD estimators not only achieve desirable convergence rates but also scale more favorably with dimensionality compared to existing minimax estimators or the commonly used add-a-constant estimator.

Future research directions include the exploration of \textit{adaptive step-size schemes} to further enhance convergence rates without compromising stability. Another promising avenue is extending the framework to \textit{online or streaming data settings}, where observations arrive sequentially and estimators must update in real-time. Additionally, incorporating \textit{structured priors} or \textit{regularization techniques} into the SMD updates could better integrate prior knowledge or enforce sparsity, particularly in overcomplete or underdetermined mixture models. Finally, evaluating the performance of our methods in practical applications such as clustering, density estimation, or language modeling would help bridge the gap between theoretical guarantees and real-world utility.

\section{Appendix}

We begin with proof of the lemma \ref{lemma SC}

\begin{proof}
    We establish convexity by directly computing the Hessian at any point \(\mathbf{m}=[m_1, \dots, m_M]\):
\begin{align*}
\lefteqn{H(m)=\nabla^2 \mathbb{E}_{P^*} \left[ \log \frac{1}{\sum_{i=1}^M m_i f_i(\zeta)} \right]}  \\[1em]
&= \mathbb{E}_{P^*} \left(
\scalebox{0.85}{$
\begin{bmatrix}
\frac{(f_1(\zeta))^2}{(\sum_{i=1}^M m_i f_i(\zeta))^2} & \dots & \frac{f_1(\zeta)f_M(\zeta)}{(\sum_{i=1}^M m_i f_i(\zeta))^2} \\
\vdots & \ddots & \vdots \\
\frac{f_M(\zeta)f_1(\zeta)}{(\sum_{i=1}^M m_i f_i(\zeta))^2} & \dots & \frac{(f_M(\zeta))^2}{(\sum_{i=1}^M m_i f_i(\zeta))^2}
\end{bmatrix}
$}
\right) \\[1em]
&= \mathbb{E}_{P^*} \left(
\scalebox{0.85}{$
\begin{bmatrix}
\frac{f_1(\zeta)}{\sum_{i=1}^M m_i f_i(\zeta)} \\
\vdots \\
\frac{f_M(\zeta)}{\sum_{i=1}^M m_i f_i(\zeta)}
\end{bmatrix}
\begin{bmatrix}
\frac{f_1(\zeta)}{\sum_{i=1}^M m_i f_i(\zeta)}, \dots, \frac{f_M(\zeta)}{\sum_{i=1}^M m_i f_i(\zeta)}
\end{bmatrix}
$} \right).
\end{align*}
This shows that the Hessian is positive definite, which implies that the objective function is convex.

To establish strong convexity, we consider the the minimum eigenvalue of the Hessian:
\begin{align*}
& \lambda_{\min}(H(\mathbf{m}))=\\
& \min_{\|v\|^2 = 1} \,
\mathbb{E}_{P^*} \Bigg[
v^\top
\scalebox{0.7}{$
\begin{bmatrix}
\frac{f_1(\zeta)}{\sum_{i=1}^M m_i f_i(\zeta)} \\
\vdots \\
\frac{f_M(\zeta)}{\sum_{i=1}^M m_i f_i(\zeta)}
\end{bmatrix}
\begin{bmatrix}
\frac{f_1(\zeta)}{\sum_{i=1}^M m_i f_i(\zeta)}, \dots, \frac{f_M(\zeta)}{\sum_{i=1}^M m_i f_i(\zeta)}
\end{bmatrix}
$}
v
\Bigg] \\
&= \min_{\|v\|^2 = 1} \,
\mathbb{E}_{P^*} \left[
\frac{\left( \sum_{i=1}^M v_i f_i(\zeta) \right)^2}
{\left( \sum_{i=1}^M m_i f_i(\zeta) \right)^2}
\right].
\end{align*}

By Assumption~\ref{assumption 2}, we can bound the denominator:
\begin{align}
&\lambda_{\min}(H(\mathbf{m})) \ge \nonumber
\\&
\scalebox{0.85}{$
\frac{1}{\left( \max_{i \in [1,\dots,M]} \sup_{\zeta \in \Omega} f_i(\zeta) \right)^2}
\min_{\|v\|^2 = 1}
\mathbb{E}_{P^*} \left[
\left( \sum_{i=1}^M v_i f_i(\zeta) \right)^2
\right]
$}
\label{eq positive}
\end{align}

From Assumption \ref{assumption 1} we know, \[\forall ~v, \|v\|^2=1~\mathbb{E}_{P^*} \left[
\left( \sum_{i=1}^M v_i, f_i(\zeta) \right)^2
\right]>0\]. 

If not, the second moment \( \sum_{i=1}^M v_i f_i(\zeta)\) is zero , which implies that \(\sum_{i=1}^M v_i f_i(\zeta)\) is zero almost surely, this means that \(\exists A ~ f_i\big|_A(\cdot)\) are linearly dependent which contradicts Assumption~\ref{assumption 1}:
\[\scalebox{.9}{$
\exists \mathcal{A} \in \mathcal{F} \ \text{such that} \ 
\int_{\mathcal{A}}
\left( \sum_{i=1}^M v_i f_i(\zeta) \right)^2
dP^*(\zeta) > 0
\quad \forall v \in \mathbb{R}^M.$}
\]

In other words, we have shown that \(\lambda_{\min}(H(m))\) is greater than a non-zero value independent of \(m\). Let us define \(\nu\):
\[
\nu := \lambda_{\min}(\nabla^2 \mathbb{E}_{P^*} \left[ \log \frac{1}{\sum_{i=1}^M m_i f_i(\zeta)} \right]) > 0.
\]
In other words, our objective function is \(\nu\)- strong convexity.

\end{proof}

We now provide the proof of Theorem~\ref{thm1}.

\begin{proof}
\textbf{Setup.} We minimize 
\[ F(m) = \mathbb{E}_{\zeta \sim P^*}\bigl[-\log \langle m, f(\zeta)\rangle\bigr] \]
over the simplex $\Delta_M$. Let $m^* = \arg\min_{m \in \Delta_M} F(m)$ and assume
$F$ is $\nu$-strongly convex on $\Delta_M$:
\[\quad \forall m, m' \in \Delta_M,
\]
\[
F(m') \geq F(m) + (m' - m)^\top \nabla F(m) + \tfrac{1}{2} \nu \|m' - m\|_2^2,
\]
equivalently
\begin{equation}\label{eq:sc-equiv}
(m' - m)^\top \bigl(\nabla F(m') - \nabla F(m)\bigr) \geq \nu \|m' - m\|_2^2.
\end{equation}
Let $\zeta_1, \zeta_2, \dots \stackrel{\text{i.i.d.}}{\sim} P^*$ and let
$G(m, \zeta) = -\nabla_m \log \langle m, f(\zeta)\rangle$ be the stochastic gradient,
so $\mathbb{E}_\zeta[G(m,\zeta)] = \nabla F(m)$. Assume the bound
\begin{equation}\label{eq:Gbound}
\mathbb{E}\bigl[\|G(m, \zeta)\|_2^2\bigr] \leq G_\infty^2 \quad \forall m \in \Delta_M.
\end{equation}
Consider the projected SGD iteration
\[
m^{t+1} = \Pi_{\Delta_M}\bigl(m^t - \gamma_t G(m^t, \zeta_t)\bigr),
\]
and define
\[
A_t = \tfrac{1}{2}\|m^t - m^*\|_2^2, \qquad a_t = \mathbb{E}[A_t].
\]

\smallskip
\textbf{Step 1: One-step inequality.}
Since $\Pi_{\Delta_M}$ is non-expansive and $\Pi_{\Delta_M}(m^*) = m^*$,
\begin{align}
A_{t+1}
&= \tfrac{1}{2}\bigl\|\Pi_{\Delta_M}(m^t - \gamma_t G(m^t,\zeta_t)) - \Pi_{\Delta_M}(m^*)\bigr\|_2^2 \nonumber \\
&\leq \tfrac{1}{2}\bigl\|m^t - \gamma_t G(m^t,\zeta_t) - m^*\bigr\|_2^2 \nonumber \\
&= A_t + \tfrac{1}{2} \gamma_t^2 \|G(m^t,\zeta_t)\|_2^2 
   - \gamma_t (m^t - m^*)^\top G(m^t,\zeta_t). \label{eq:onestep}
\end{align}

\smallskip
\textbf{Step 2: Conditioning.}
Let $\mathcal{F}_{t-1} = \sigma(\zeta_1, \dots, \zeta_{t-1})$. Since $m^t$ is
$\mathcal{F}_{t-1}$-measurable and $\zeta_t$ is independent of $\mathcal{F}_{t-1}$,
\begin{align}
&\mathbb{E}\bigl[(m^t - m^*)^\top G(m^t, \zeta_t)\bigr] \nonumber \\
&= \mathbb{E}\bigl[(m^t - m^*)^\top \mathbb{E}[G(m^t,\zeta_t)\mid \mathcal{F}_{t-1}]\bigr] \nonumber \\
&= \mathbb{E}\bigl[(m^t - m^*)^\top \nabla F(m^t)\bigr]. \label{eq:cond}
\end{align}
Taking expectations in \eqref{eq:onestep} and using \eqref{eq:Gbound} and \eqref{eq:cond},
\begin{equation}\label{eq:rec0}
a_{t+1} \leq a_t - \gamma_t\, \mathbb{E}\bigl[(m^t - m^*)^\top \nabla F(m^t)\bigr] 
+ \tfrac{1}{2}\gamma_t^2 G_\infty^2.
\end{equation}

\smallskip
\textbf{Step 3: Strong convexity at $m^*$.}
From \eqref{eq:sc-equiv} with $m' = m^t$, $m = m^*$,
\[
(m^t - m^*)^\top \bigl(\nabla F(m^t) - \nabla F(m^*)\bigr) \geq \nu\|m^t - m^*\|_2^2.
\]
By first-order optimality of $m^*$ on $\Delta_M$, $(m - m^*)^\top \nabla F(m^*) \geq 0$
for all $m \in \Delta_M$; in particular $(m^t - m^*)^\top \nabla F(m^*) \geq 0$.
Adding,
\begin{equation}\label{eq:scbound}
(m^t - m^*)^\top \nabla F(m^t) \geq \nu\|m^t - m^*\|_2^2 = 2\nu\, A_t.
\end{equation}
Substituting \eqref{eq:scbound} into \eqref{eq:rec0},
\begin{equation}\label{eq:rec}
a_{t+1} \leq (1 - 2\nu\gamma_t)\, a_t + \tfrac{1}{2}\gamma_t^2 G_\infty^2.
\end{equation}

\textbf{Step 4: Solving the recursion.}
Pick the step size $\gamma_t = \dfrac{2}{\nu(t+1)}$. Substituting into \eqref{eq:rec},
\[
a_{t+1} \;\leq\; \frac{t-3}{t+1}\, a_t \;+\; \frac{2G_\infty^2}{\nu^2(t+1)^2}.
\]
Let
\[
B := \max\!\left\{\, a_0,\ \frac{4G_\infty^2}{\nu^2} \,\right\}.
\]
We claim $a_t \leq \dfrac{B}{t+1}$ for all $t \geq 0$, by induction.

\emph{Base case ($t=0$):} $a_0 \leq B$ by definition of $B$.

\emph{Inductive step:} suppose $a_t \leq B/(t+1)$. The coefficient $(t-3)/(t+1)$ is negative for $t \leq 2$ and non-negative for $t \geq 3$, so we split cases.

\emph{Case $t \in \{0,1,2\}$.} Since $a_t \geq 0$ and the coefficient is non-positive, $\frac{t-3}{t+1}\, a_t \leq 0$, hence
\[
a_{t+1} \;\leq\; \frac{2G_\infty^2}{\nu^2(t+1)^2} \;\leq\; \frac{B}{2(t+1)^2} \;\leq\; \frac{B}{t+2},
\]
where the second inequality uses $4G_\infty^2/\nu^2 \leq B$, and the third uses $t+2 \leq 2(t+1)^2$ for $t \geq 0$.

\emph{Case $t \geq 3$.} By the inductive hypothesis $a_t \leq B/(t+1)$,
\begin{align*}
a_{t+1}
&\leq \frac{(t-3)\,B}{(t+1)^2} + \frac{2G_\infty^2}{\nu^2(t+1)^2}
   \;\leq\; \frac{(t-3)\,B + B/2}{(t+1)^2} \\
&\leq \frac{(t-2)\,B}{(t+1)^2}
   \;\leq\; \frac{B}{t+2},
\end{align*}
using $2G_\infty^2/\nu^2 \leq B/2$ and $(t-2)(t+2) \leq (t+1)^2$.

\smallskip
\textbf{Conclusion.} Setting $t = N$,
\[
\mathbb{E}\bigl[\|m^N - m^*\|_2^2\bigr] \;=\; 2\,a_N \;\leq\; \frac{2B}{N+1}
\;=\; \]\[\frac{1}{N+1}\,
\max\!\left\{\, \|m^0 - m^*\|_2^2,\ \frac{8\,G_\infty^2}{\nu^2} \,\right\}.
\]

\end{proof}

We provide the proof of Theorem~\ref{thm2} here.

\begin{proof}
\textbf{Setup.} We minimize $F(m) = \mathbb{E}_{\zeta \sim P^*}[-\log \langle m, f(\zeta) \rangle]$ over the  simplex

\medskip
\textbf{Step 1: Optimality of the prox step.} The first-order optimality condition for (\ref{eq:updaterule}) gives, for any $m^* \in \Delta_M$,
\begin{align} \label{eq:opt}
\bigl\langle \gamma_t G(m^t; \zeta_t) + \nabla \Phi(m^{t+1}) - \nabla \Phi(m^t),\ m^{t+1} - m^* \bigr\rangle \leq 0.
\end{align}

\textbf{Step 2: Bregman three-point identity.} For any $u \in \Delta_M$,
\begin{align} \label{eq:3pt}
&\bigl\langle \nabla \Phi(m^{t+1}) - \nabla \Phi(m^t),\ u - m^{t+1} \bigr\rangle 
= \nonumber \\ &D_{KL}(u \| m^t) - D_{KL}(u \| m^{t+1}) - D_{KL}(m^{t+1} \| m^t).
\end{align}
Combining \eqref{eq:opt} with $u = m^*$ and \eqref{eq:3pt}:
\begin{align} \label{eq:star}
&\gamma_t \langle G(m^t; \zeta_t),\ m^{t+1} - m^* \rangle 
\leq  \nonumber \\ &D_{KL}(m^* \| m^t) - D_{KL}(m^* \| m^{t+1}) - D_{KL}(m^{t+1} \| m^t).
\end{align}

\textbf{Step 3: Splitting and Young's inequality.} Decompose
\begin{align*}
&\langle G(m^t; \zeta_t),\ m^{t+1} - m^* \rangle = \nonumber \\ &\langle G(m^t; \zeta_t),\ m^t - m^* \rangle - \langle G(m^t; \zeta_t),\ m^t - m^{t+1} \rangle. 
\end{align*}
By H\"older's and Young's inequalities,
\begin{align}
&\gamma_t \langle G(m^t; \zeta_t),\ m^t - m^{t+1} \rangle \nonumber \\
&\leq \gamma_t \|G(m^t; \zeta_t)\|_\infty \|m^t - m^{t+1}\|_1 \nonumber \\
&\leq \frac{\gamma_t^2 \|G(m^t; \zeta_t)\|_\infty^2}{2} + \frac{1}{2} \|m^t - m^{t+1}\|_1^2.
\end{align}
From assumption \ref{assumption 2} and by Pinsker's inequality, $\frac{1}{2} \|m^t - m^{t+1}\|_1^2 \leq D_{KL}(m^{t+1} \| m^t)$, so
\begin{align} \label{eq:starstar}
\gamma_t \langle G(m^t; \zeta_t),\ m^t - m^{t+1} \rangle 
\leq \frac{\gamma_t^2 G_\infty^2}{2} + D_{KL}(m^{t+1} \| m^t).
\end{align}
Substituting \eqref{eq:starstar} into \eqref{eq:star}, the $D_{KL}(m^{t+1} \| m^t)$ terms cancel:
\begin{align} \label{eq:starstarstar}
&\gamma_t \langle G(m^t; \zeta_t),\ m^t - m^* \rangle 
\leq \nonumber\\ &D_{KL}(m^* \| m^t) - D_{KL}(m^* \| m^{t+1}) + \frac{\gamma_t^2 G_\infty^2}{2}.
\end{align}

\textbf{Step 4: Taking expectations and applying strong convexity condition.} Since $m^t$ is $\mathcal{F}_{t-1}$-measurable and $G(m^t; \zeta_t)$ is conditionally unbiased,
\[
\mathbb{E}\bigl[\langle G(m^t; \zeta_t),\ m^t - m^* \rangle \mid \mathcal{F}_{t-1}\bigr] = \langle \nabla F(m^t),\ m^t - m^* \rangle.
\]
Applying strong convexity condition,
\begin{align*}
&\langle \nabla F(m^t),\ m^t - m^* \rangle \geq \\ &F(m^t) - F(m^*) + \nu\, D_{KL}(m^* \| m^t) \geq \\ &\nu\, D_{KL}(m^* \| m^t),
\end{align*}
where the last inequality uses $F(m^t) \geq F(m^*)$ since $m^*$ is the minimizer. Taking total expectations in \eqref{eq:starstarstar},
\begin{align} \label{eq:recursion}
\mathbb{E}[D_{KL}(m^* \| m^{t+1})] \leq (1 - \gamma_t \nu)\, \mathbb{E}[D_{KL}(m^* \| m^t)] + \frac{\gamma_t^2 G_\infty^2}{2}.
\end{align}

\textbf{Step 5: Solving the recursion.} Let $a_t := \mathbb{E}[D_{KL}(m^* \| m^t)]$ and $C := G_\infty^2 / 2$. Choose the step size $\gamma_t = \frac{2}{\nu(t+1)}$ and define
\[
B := \max\!\left\{ a_0,\ \frac{4C}{\nu^2} \right\}.
\]
We prove by induction that $a_t \leq \frac{B}{t+1}$ for all $t \geq 0$.

\emph{Base case ($t = 0$):} $a_0 \leq B$ by definition of $B$.

\emph{Inductive step:} Assume $a_t \leq \frac{B}{t+1}$. Then from \eqref{eq:recursion},
\begin{align}
a_{t+1} &\leq \left(1 - \frac{2}{t+1}\right) \frac{B}{t+1} + \frac{4C}{\nu^2(t+1)^2} \nonumber \\
&= \frac{B(t-1)}{(t+1)^2} + \frac{4C}{\nu^2(t+1)^2} \nonumber \\
&\leq \frac{B(t-1) + B}{(t+1)^2} = \frac{Bt}{(t+1)^2},
\end{align}
where the last inequality uses $4C/\nu^2 \leq B$. Since $t(t+2) \leq (t+1)^2$, we have $\frac{Bt}{(t+1)^2} \leq \frac{B}{t+2}$, completing the induction.

\textbf{Conclusion.} Setting $t = N$,
\begin{align}
\mathbb{E}_{P^*}[D_{KL}(m^* \| m^N)] \leq &\frac{1}{N+1}\, \max\!\left\{ D_{KL}(m^* \| m^0),\ \frac{2 G_\infty^2}{\nu^2} \right\}.
\end{align}
\end{proof}

\end{document}